\documentclass[letterpaper]{article} % DO NOT CHANGE THIS
\usepackage{aaai23}  % DO NOT CHANGE THIS
\usepackage{times}  % DO NOT CHANGE THIS
\usepackage{helvet}  % DO NOT CHANGE THIS
\usepackage{courier}  % DO NOT CHANGE THIS
\usepackage[hyphens]{url}  % DO NOT CHANGE THIS
\usepackage{graphicx} % DO NOT CHANGE THIS
\usepackage{natbib}  % DO NOT CHANGE THIS AND DO NOT ADD ANY OPTIONS TO IT
\usepackage{caption} % DO NOT CHANGE THIS AND DO NOT ADD ANY OPTIONS TO IT
\DeclareCaptionStyle{ruled}{labelfont=normalfont,labelsep=colon,strut=off} % DO NOT CHANGE THIS
\usepackage{algorithm}
\usepackage[noend]{algpseudocode} % The noend option will discard end of loop lines like end if, end for, end procedure, etc.

\usepackage[dvipsnames]{xcolor}

\usepackage{newfloat}
\usepackage{listings}
\floatstyle{ruled}
\newfloat{listing}{tb}{lst}{}
\floatname{listing}{Listing}
\usepackage{svg}
\usepackage{hhline} % RVAdded
\usepackage{multirow}
\usepackage{placeins}
\usepackage{amsthm}
\usepackage{amssymb}
\newtheorem{theorem}{Theorem}
\newtheorem{lemma}[theorem]{Lemma}

\title{Effective Integration of Weighted Cost-to-go and Conflict Heuristic within Suboptimal CBS}
\author{
    Rishi Veerapaneni, Tushar Kusnur, Maxim Likhachev
}
\affiliations{
    Robotics Institute, Carnegie Mellon University \\
    \{rveerapa, tkusnur, mlikhach\}@andrew.cmu.edu
}

\begin{document}

\maketitle

\begin{abstract}
Conflict-Based Search (CBS) is a popular multi-agent path finding (MAPF) solver that employs a low-level single agent planner and a high-level constraint tree to resolve conflicts. 
The vast majority of modern MAPF solvers focus on improving CBS by reducing the size of this tree through various strategies with few methods modifying the low level planner.
Typically low level planners in existing CBS methods use an \textit{unweighted} cost-to-go heuristic, with suboptimal CBS methods also using a conflict heuristic to help the high level search.
In this paper, we show that, contrary to prevailing CBS beliefs, a  \textit{weighted} cost-to-go heuristic can be used effectively alongside the conflict heuristic in two possible variants.
In particular, one of these variants can obtain large speedups, 2-100x, across several scenarios and suboptimal CBS methods. Importantly, we discover that performance is related not to the weighted cost-to-go heuristic but rather to the \textit{relative conflict heuristic weight}'s ability to effectively balance low-level and high-level work.
Additionally, to the best of our knowledge, we show the first theoretical relation of prioritized planning and bounded suboptimal CBS and demonstrate that our methods are their natural generalization.
\textbf{Update March 2024: We found that the relative speedup decreases to around 1.2-10x depending on how the conflict heuristic is computed (see appendix for more details).}

% Conflict-Based Search (CBS) is a popular multi-agent path finding (MAPF) solver that employs a low-level single agent planner and a high-level constraint tree to resolve conflicts. 
% The vast majority of modern MAPF solvers focus on improving CBS by reducing the size of this tree through various strategies with few methods modifying the low level planner.
% Typically low level planners in existing suboptimal CBS methods use an \textit{unweighted} cost-to-go heuristic (to speed-up low level search) and a  conflict heuristic (to help the high level search). We explore two variants for efficiently weighting these heuristics in suboptimal CBS and have three main findings.
% First and most important, we show that, missing from existing CBS knowledge, performance is heavily dictated by the conflict heuristic's role in trading off low-level and high-level work. We can explicitly balance this workload via a \textit{relative conflict heuristic weight} and obtain large speedups, 2-100x.
% Second, both our variants show how a weighted cost-to-go heuristic can be used effectively alongside the conflict heuristic.
% Finally, to the best of our knowledge, we show the first theoretical relation of prioritized planning and bounded suboptimal CBS and demonstrate that our methods are their natural generalization.
\end{abstract}

\section*{Introduction}\label{sec:intro}

\begin{figure}[t]
    \centering
    \includegraphics[width=2.3in]{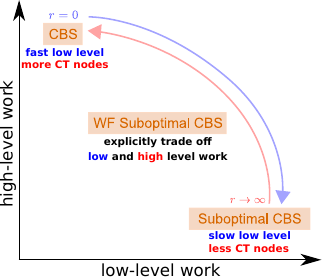}
    \caption{\footnotesize CBS's original low-level planner finds shortest paths ignoring additional conflicts created, resulting in a large amount of CT nodes for the high-level search to resolve the conflicts. Suboptimal CBS (ECBS, EECBS) utilize a low-level focal search that minimizes future conflicts alongside path length, resulting in a slower low level planner but substantially reducing the high level work and boosting performance. Our Weighted Focal method improves performance by finding a sweet spot in the middle. We introduce a hyper-parameter $r$ in the low-level focal search which controls trading off low level work (finding paths) and high level work (resolving conflicts), allowing us to find low-level plans faster with a little more high level work than current suboptimal CBS methods.}
    \label{fig:illustrative-example}
    \vspace*{-2mm}
\end{figure}

Multi-Agent Path Finding (MAPF) is the problem of computing collision-free paths for a team of agents in a known environment while minimizing a measure of their travel times.
This is required for several real-world tasks such as the smooth operation of automated warehouses~\cite{li2020lifelong}, robot soccer~\citep{biswas2014opponent}, collaborative manufacturing~\cite{sun2002adaptive}, coverage~\cite{kusnur2021planning}, and others.
MAPF is a challenging problem and is shown to be NP-complete~\citep{ratner1986finding}.

Prioritized Planning (PP)~\citep{erdmann1987multiple} is a fast multi-agent planning approach that sequentially plans agents avoiding earlier agents with better ``priority", and has been applied to several domains ~\citep{wu2020multi,vcap2015prioritized,velagapudi2010decentralized}. However PP provides no guarantees on completeness or bounded suboptimality.
% Prioritized Planning (PP)~\citep{erdmann1987multiple} is a fast multi-agent planning approach that scales well to large numbers of agents with extensions that have been applied to several domains ~\citep{wu2020multi,vcap2015prioritized,velagapudi2010decentralized}. However PP provides no guarantees on completeness or bounded suboptimality of the overall solution.

Conflict-Based Search (CBS) is a popular complete and optimal MAPF solver that employs a low-level single agent planner and a high-level constraint tree (CT) to resolve conflicts.
Several methods speed up CBS by reducing the CT size by explicitly pruning branches, selectively expanding branches, adding sets of constraints, detecting symmetries, and improving high-level heuristics~\citep{boyarski2015icbs, fawarecp,li2019improved,li2020new,srli2021}.

Enhanced CBS (ECBS)~\citep{barer2014suboptimal} introduced the first bounded-suboptimal version of CBS, utilizing a focal search on the high level as well as another focal search planner on the low level that minimizes path conflicts with other agents and therefore decreases the CT size. ECBS specifically mentions how modifying the low level planner to use a weighted cost-to-go heuristic returns paths with many conflicts, leading to a larger CT tree and proved ``ineffective in [their] experiments" as direct motivation for reducing the path conflits instead. Explicit Estimation CBS (EECBS)~\citep{li2021eecbs} replaces ECBS's high level focal search with Explicit Estimation Search~\citep{thayer2011bounded}  but keeps the same low level focal search. Continuous-time CBS (CCBS)~\citep{continuouscbs} incorporates Safe Interval Path Planning (SIPP)~\citep{Phillips2011SIPPSI} to speed up the low level search by reasoning about waits but also does not employ a weighted heuristic.
% \citep{ma2019searching} introduced agent prioritization within CBS by developing CBSw/P and PBS which search over agent orderings in best-first and depth-first manners respectively at the loss of bounded sub-optimality and completeness. 
To the authors' best knowledge, no prior work has effectively used a weighted cost-to-go heuristic in any manner in the CBS framework, with the prevailing norm that doing so would lead to more conflicts, reduce performance, or remove bounded sub-optimality.

Our initial insight is that we can use the weighted cost-to-go heuristic along with the conflict heuristic. We introduce the first bounded sub-optimal CBS methods that incorporate a weighted cost-to-go heuristic with the conflict heuristic within CBS's single agent planner. 
We find that performance (and large performance gains) is controlled by the relative weight of the cost-to-go and conflict heuristic which allows us to trade off low-level and high-level work, see Figure \ref{fig:illustrative-example}. This runs counter to the experience of researchers familiar with single agent planning where we would expect the weighted cost-to-go heuristic to improve performance by finding paths faster. 
% \textit{However, contrary to the experience of researchers familiar with single agent planning, our experimental analysis shows that a weighted cost-to-go heuristic does not improve performance by finding paths faster}. Instead, we find that performance (and large performance gains) is dictated by the relative weight of the cost-to-go and conflict heuristic which trades off low-level and high-level work.
Our contributions are:

\begin{enumerate}
    \item Incorporating the weighted cost-to-go heuristic in the open queue, and studying how the path lower bounds interact with certain CBS improvements.
    \item Combining the weighted cost-to-go heuristic with a weighted conflict heuristic ratio in the focal queue and demonstrating that the relative weight, not the weighted heuristic, dictates performance. We provide additional analysis and show how this behaviour is novel in respect to existing CBS and single-agent planning intuition. 
    \item Reducing PP to a particular step of suboptimal CBS and showing that our methods are the natural generalization. 
    % We recommend practitioners using PP switch to our methods as they can get the same initial prioritized planning behaviour while still being complete.
    % \item Extensive experiments showcasing the effects of the various parameters, and discovering an important relationship between the cost-to-go and collision heuristic. 
\end{enumerate}

\section*{Incorporating Weighted Cost-to-go Heuristic}
% CBS employs a low level single agent planner to plan paths for individual agents, and a high level constraint tree planner that plans over constraints. Bounded sub-optimal CBS methods (e.g. ECBS, EECBS) modify the single agent planner via a focal queue to also minimize the number of collisions with other agents (which reduces future constraints in the CT). We build upon EECBS as it was shown to outperform ECBS and other MAPF planners, but note that our method is directly usable in ECBS and any other bounded sub-optimal CBS planner using a low level focal planner. Our main idea is to incorporate a weighted cost-to-go heuristic in the single agent planner along with the collision heuristic. The user's sub-optimality hyper-parameter $w_{so}$ is assumed to be fixed and outside our optimization.

% CBS employs a low level single agent planner to plan paths for individual agents, and a high level constraint tree planner that plans over constraints. 
CBS utilizes an optimal space-time A* low level planner with a precomputed cost-to-go heuristic that measures the optimal distance to goal ignoring conflicts. Bounded sub-optimal CBS methods (e.g. ECBS, EECBS) modify the single agent planner to a focal search that computes $w_{so}$ sub-optimal path that minimizes the number of conflicts with other agents (which reduces future constraints in the CT). 
The low level planner must also return a lower bound on the optimal solution cost which is required for certain CBS improvements, specifically prioritized conflicts and symmetry reasoning, see \citet{li2021eecbs} for full justification. 
% , which require agents to be on their optimal path (which occurs when the returned lower bound equals the cost of the returned path)
The low level focal search has two queues; OPEN which searches over optimal paths (paths sorted by cost) and maintains an optimality bound, and FOCAL which prioritizes $w_{so}$ sub-optimal paths with fewer conflicts (paths sorted by conflicts). We specifically discuss our method in relation to EECBS as it was shown to outperform ECBS and other MAPF planners, but our method is directly usable in ECBS and any other bounded sub-optimal CBS planner using a low level focal planner (see Table \ref{tab:suboptimality-joined-summary}). 

Our main idea is to incorporate a $w_h$ weighted cost-to-go heuristic in the single agent planner in two ways: one in OPEN independent of the conflict heuristic and the other in FOCAL along with the conflict heuristic. Algorithm \ref{alg:ecbs-low-level} showcases EECBS's general low level search pseudocode, with W-EECBS changes highlighted in blue. Ties in FOCAL are broken by $f_{open}$. The user's suboptimality hyper-parameter $w_{so}$ is assumed to be fixed and outside our optimization.

\textbf{Intuition:} Our intuition is that by incorporating the cost-to-go weight with the conflict heuristic, we should be able to plan faster while also minimizing conflicts. This intuition is informed  by single-agent planning where weighting the cost-to-go heuristic speeds up search. We therefore think that increasing $w_h$ alongside using the conflict heuristic in FOCAL will result in the low-level planner finding similar quality paths faster without many additional conflicts, resulting in speed ups. However, we find that increasing $w_h$ does \textit{not} help performance significantly. Instead we find that modulating the conflict heuristic's relative weight in FOCAL substantially improves overall performance. 
% To our knowledge, this finding on the conflict heuristic's nuanced role in Suboptimal CBS has never been discussed before.

% This weighting introduces a second parameter (on top of ECBS's suboptimality parameter) that the user can change.
% We introduce two ways of incorporating a weighted single agent planner, one by weighting the open queue and another by changing the focal priority. Algorithms \ref{alg:ecbs-low-level} showcase EECBS's general low level search psued-code, with functions W-EECBS changes highlighted in blue. Ties in the focal queue are broken by $f_{open}$.

\begin{algorithm}[htb]
\caption{Suboptimal CBS low level focal search planner}
\label{alg:ecbs-low-level}
\textbf{Input}: $n_{start}$, atGoal(), Paths $P_I$ of other agents \\
\textbf{Output}: Lower bound $LB$ on optimal path cost, Path from $n_{start}$ with sub-optimality $\leq w_{so}$ (i.e. cost $\leq w_{so}*LB$)
\begin{algorithmic}[1] %[1] enables line numbers
\State \textcolor{blue}{Set$W_f$()}
\State OPEN = FOCAL = $\{n_{start}\}$, $LB = F_{best} = 0$
\While{FOCAL $\neq \emptyset $}
    \State $n \gets$ FOCAL.pop()
    \State OPEN.remove($n$)
    \State $LB \gets $ $\max(LB,$ \textcolor{blue}{UpdateLowerBound()}$)$
    \If{atGoal($n$)}
        \State \textbf{return} $LB$, Solution backtracking from $n$
    \EndIf
    \For{$n'$ $\in$ succ($n$)}
        % \Statex \Comment{Always insert n to open}
        \State $g \gets n.g + cost(n,n')$
        \State $h \gets$ getCostToGoHeuristic($n'$)
        \State $n'.F_{open} \gets$ \textcolor{blue}{$f_{open}(g, h)$}
        \State OPEN.insert($n'$)
        \State $c \gets$ getNumConflictsFromPaths($n'$, $P_I$)
        \State $n'.F_{focal} \gets$ \textcolor{blue}{$f_{focal}(g, h, c)$}
    \EndFor
    
    \Statex \Comment{Update FOCAL}
    \State $F_{best} \gets min_{k \in OPEN}\ k.F_{open}$
    \ForAll {$n' \in OPEN, \notin FOCAL $} 
        \If {$n'.F_{open} \leq \textcolor{blue}{w_f}*F_{best}$}
            \State FOCAL.insert($n'$)
        \EndIf
    \EndFor
\EndWhile
\State \textbf{return} $NaN$, No solution
\Statex
\Procedure{${f}_{open}$}{$g,h$}:
    \State \textbf{return} $g+h$
\EndProcedure
\Procedure{$f_{focal}$}{$g,h,c$}:
    \State \textbf{return} $c$
\EndProcedure
\Procedure{Set$W_f$()}{}:
    \State $w_f \gets w_{so}$
\EndProcedure
\Procedure{UpdateLowerBound()}{}:
    \State \textbf{return} $F_{best}$
\EndProcedure
\end{algorithmic}
\end{algorithm}

\subsection{Weighted Open Variant (WO-EECBS)}
% The weighted hueristic is incorporated in the
OPEN's priority function is weighted by $w_h$, while FOCAL remains unchanged, prioritized by the number of conflicts. To maintain our overall suboptimality bound, the focal bound $w_f$ is scaled to $w_{so}/w_h$ which constrains $w_h \in [1,w_{so}]$ as we need $w_f \geq 1$. Since the f-values in OPEN are now weighted by $w_h$, we obtain a lower bound on the optimal path cost by scaling the minimum f-value in OPEN, $F_{best}$, to $F_{best}/w_h$. Note that $w_h=1$ trivially results in regular EECBS.

One side effect of this method is that this naively computed lower bound is usually substantially lower than the optimal path cost even though the path may not have been very sub-optimal.
Several papers have discussed this pessimistic lower bound in weighted A* single agent search \cite{optimisticsearch, improvedLowerBound}.
This pessimistic lower bound should then theoretically reduce the amount of prioritized conflicts (PC) and symmetry reasoning (SR) applied. We therefore \textit{a posteriori} compute a better lower bound using \citet{improvedLowerBound} and test if this increases the usage of PC and SR, and boosts performance.

\begin{algorithm}[htb]
\caption{Weighted Open modifications}
\label{alg:algorithm}
\textbf{Parameters}: Cost-to-go heuristic weight $w_h$
\begin{algorithmic}[1] %[1] enables line numbers
\Procedure{${f}_{open}$}{$g,h$}:
    \State \textbf{return} $g+w_h*h$
\EndProcedure
\Procedure{Set$W_f$()}{}:
    \State $w_f \gets w_{so}/w_h$
\EndProcedure
\Procedure{GetLowerBound()}{}: \Comment{Naive}
    \State  \textbf{return} $F_{best}/w_h$
\EndProcedure
\Procedure{GetLowerBound()}{}: \Comment{Improved}
    \State $g_{min} \gets min_{n \in OPEN}\ n.g$
    \State \textbf{return} $(F_{best} + (w_h-1)*g_{min})/w_h$
\EndProcedure
\end{algorithmic}
\end{algorithm}

\begin{algorithm}[htb]
\caption{Weighted Focal modifications}
\label{alg:algorithm}
\textbf{Parameters}: $w_h$, Relative conflict weight  $r$
\begin{algorithmic}[1] %[1] enables line numbers
\Procedure{$f_{focal}$}{$g,h,c$}:
    \State \textbf{return} $g+w_h*(h+r*c)$
\EndProcedure
\end{algorithmic}
\end{algorithm}

\subsection{Weighted Focal Variant (WF-EECBS)}
We keep OPEN unweighted and instead incorporate the weighted heuristic in FOCAL along with the inadmissible conflict heuristic. This requires us to balance the importance of these competing heuristics in FOCAL's priority function $g+w_h*(h+r*c)$ with $w_h \geq 1, r \geq 0$.
Manipulating $r$ changes the relative importance of finding a solution fast (lower $r$) vs avoiding conflicts (higher $r$). Note that $w_h=1$ and $r \to \infty$ results in regular EECBS (preferring paths with lowest conflicts). 
% Therefore this method generalizes the idea introduces in ECBS. 
Due to the use of FOCAL, $w_h$ (and $r$) can be arbitrarily large and is not bounded by $w_{so}$ unlike WO-EECBS. In our experiments we see that WF-EECBS significantly outperforms WO-EECBS and EECBS, therefore Weighted EECBS (W-EECBS) refers to this weighted focal version.

\begin{lemma} WO-EECBS and WF-EECBS are both $w_{so}$ sub-optimal.
\end{lemma}
\begin{proof}
EECBS's overall optimality is split between the high-level CT sub-optimal search and the low-level sub-optimal search. Since the high-level search is unchanged and identical to EECBS, we just need to prove that WO-EECBS and WF-EECBS have the same low-level sub-optimality $w_{so}$ as EECBS.

% The focal queue returns a node at most $w_f$ sub-optimal in respect to the open queue, which may have its own sub-optimality.
In WO-EECBS: FOCAL returns a node at most $w_f$ sub-optimal compared to OPEN which is weighted by $w_h$. Our overall optimality is then $w_f*w_h = w_{so}/w_h*w_h = w_{so}$.

In WF-EECBS: FOCAL's sub-optimality is fixed regardless of $f_{focal}$, and OPEN is optimal, so our overall optimality is trivially $w_f = w_{so}$.
\end{proof}

\subsection{Relating CBS, Prioritized Planning, and W-EECBS}
CBS-based algorithms and PP are usually treated as distinct categories of MAPF search based methods. \citet{ma2019searching} introduces priorities in CBS as a distinction to regular CBS and \citet{li2022mapf-lns2} employs a modified PP planner that return paths with least conflicts, but neither attempt to relate PP and CBS.

Here we prove that PP is actually equivalent to the first step of generating the initial agent paths in the root CT node in EECBS (and other bounded sub-optimal CBS planners like ECBS) with an infinite sub-optimality. With $w_{so} = \infty$ in EECBS, all states in OPEN in the single agent planner are inserted into FOCAL, and therefore expansions are sorted first by their number of conflicts, and then the path f-value. In the root CT node, agents will try to avoid all previous agents and search over all conflict=0 paths, then conflict=1 after exhausting all conflict=0 paths, then conflict=2, etc. This first step is identical to PP; EECBS with $w_{so} = \infty$ differs only in its ability to continue planning over conflicts while PP fails in that scenario. To the authors' knowledge, this is the first time there has been an explicit relation between sub-optimal CBS and PP. WO-EECBS and WF-EECBS are the two generalized methods combining the weighted low-level planner commonly used in PP with EECBS's conflict resolution mechanism.

\section*{Experimental results}
% We test our methods with different numbers of agents on 8 diverse maps and report mean values across 5 seeds. 
We test our methods with different numbers of agents, in increments of 50, on 8 diverse maps (titled in each plot) from \citet{stern2019mapfbenchmark} and report the mean values across 5 seeds. Table \ref{tab:maps-stats} shows the diversity of the maps; plots contain the maps in the same order sorted by decreasing free states. 

\begin{table}[h]
\centering
\resizebox{0.45\textwidth}{!}{
\begin{tabular}{|c|c|c|c|}
\hline
Map name & Max \# agents & Raw Area & \# free states \\ \hhline{|=|=|=|=|}
Paris\_1\_256 & 1000 & 256x256 & 47240 \\ \hline
den520d & 1000 & 256x257 & 28178 \\ \hline
ht\_chantry & 1000 & 162x141 & 7461 \\ \hline
den312d & 1000 & 65x81 & 2445 \\ \hline
empty-48-48 & 1000 & 48x48 & 2304 \\ \hline
empty-32-32 & 500 & 32x32 & 1024 \\ \hline
random-32-32-10 & 450 & 32x32 & 922 \\ \hline
random-32-32-20 & 400 & 32x32 & 819 \\ \hline
\end{tabular}}
\caption{ \footnotesize
\textbf{Map statistics |} We show the maximum number of agents, height by width raw area, and number of free states on each of the eight maps that we use to evaluate our methods. Figures are sorted in this same order horizontally (top left subplot will be the largest map, bottom right will be the smallest), to showcase the relationship of performance with map size.}
\label{tab:maps-stats}
\end{table}

% We use $w_{so}=2$ and a timeout of $300$ seconds in all our experiments except Figure \ref{fig:weightedFocal-ratio-full} which has a $60$ second timeout.
We use $w_{so}=2$ and a timeout of $300$ seconds in all our experiments unless otherwise specified.
In all figures, if a method failed (timed out on all 5 seeds) on a particle number of agents on a map, we do not report larger number of agents, see Appendix Section \ref{section:appendix-justify-timeouts} for full justification. The speed up $S_{method} = T_{baseline}/T_{method}$ is reported to normalize differences in hardware, where the baseline is the unweighted method (ECBS
or EECBS) based on context. In all tables, speeds up are computed only on instances where the baseline did not timeout.

% We provide a short summary of each figure:

% Figure \ref{fig:weightedAnchorLBImprovements}: Showcases how improving the lower bound increases the utilization of cardinal conflicts and symmetry reasoning in WO-EECBS. 

% Figure \ref{fig:weightedAnchorResults}: Demonstrates WO-EECBS's varying speed ups across several maps, as well as shows the surprising negative effect of improving the lower bound on overall performance.

% Figure \ref{fig:weightedFocal-ratio-full}: Reveals that the ratio between the conflict and cost-to-go weight parameters dictates performance in WF-EECBS. 

% Figure \ref{fig:weightedFocal-regular-full}: Highlights WF-EECBS's large performance gains on large maps, and how weighting the cost-to-go heuristic generally helps obtains larger speed ups in larger maps and low-medium conflict regimes (which matches expectations as in high conflict regimes the cost-to-go heuristic will be less informative).

% Figure \ref{fig:cbspp-node-full}: Validates how WO-EECBS with a very large sub-optimality is equivalent to prioritized planning in the root node except with EECBS's conflict resolution mechanism.

% Figure \ref{fig:cbspp-success-full}: Exhibits the increased success rate of using W-EECBS with a very large sub-optimality over prioritized planning, due to W-EECBS ability to resolve conflicts in high conflict regimes.

\begin{figure*}[p]
    \centering
    \includegraphics[width=1\textwidth]{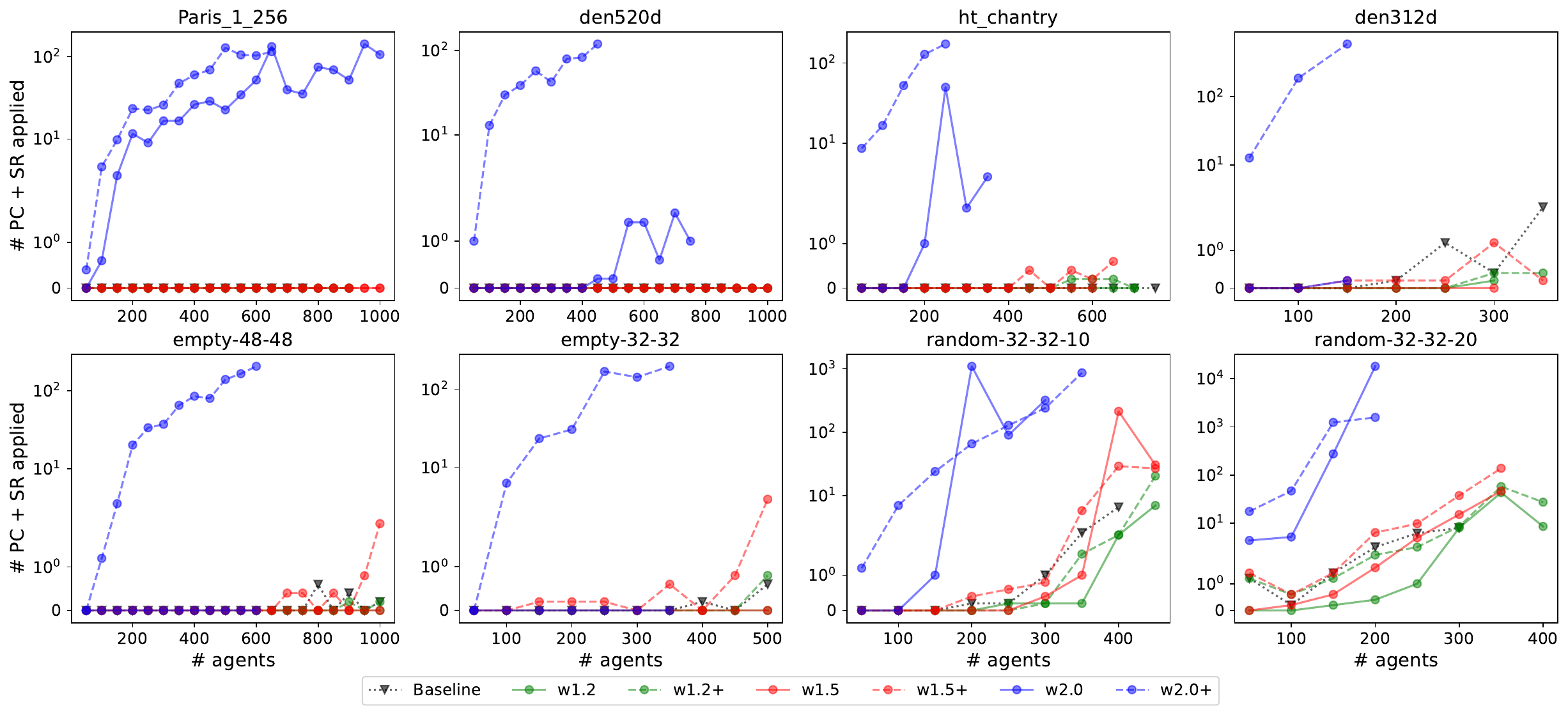}
    \caption{\textbf{Viewing the effect of improving the lowerbound on the WO-EECBS |} The ``+" label denotes using an improved lowerbound; improving the lowerbound leads to a significant higher usage of CBS improvements with the y-axis denoting the average number (across 5 seeds) of cardinal conflicts and symmetry reasoning applied for each problem instance. Without the improved lower bound, WO-EECBS is usually unable to use these CBS improvements. Methods terminate on a map once they fail all 5 seeds on a certain number of agents or they reach the maximum number of agents in a scene, fractional values are due to the averaging over 5 seeds.
    }
    \label{fig:weightedAnchorLBImprovements}
\end{figure*}

\begin{figure*}[p]
    \centering
    \includegraphics[width=1\textwidth]{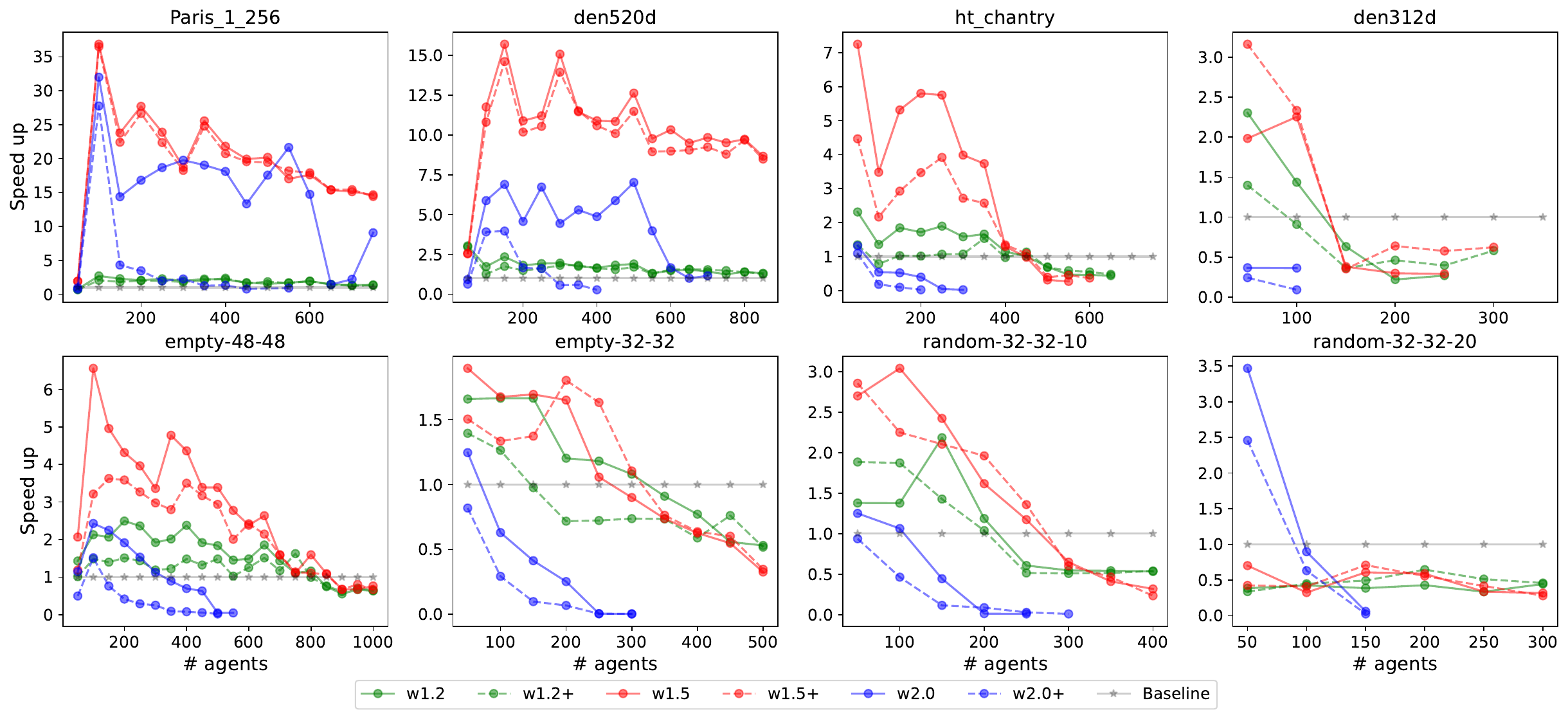}
    \caption{\textbf{WO-EECBS full results |} The ``+" label denotes using an improved lowerbound. A medium weighted value of w=1.5 performs the best on both maps. However, the maximum speed up peaks at around 35 in the Paris scenario and it struggles in all the harder scenarios (e.g. both random maps, empty-32-32, den312d). Additionally the improved lowerbound actually decreases performance contrary to our expectations.}
    \label{fig:weightedAnchorResults}
\end{figure*}

\begin{figure*}[p]
    \centering
    \includegraphics[width=0.98\textwidth]{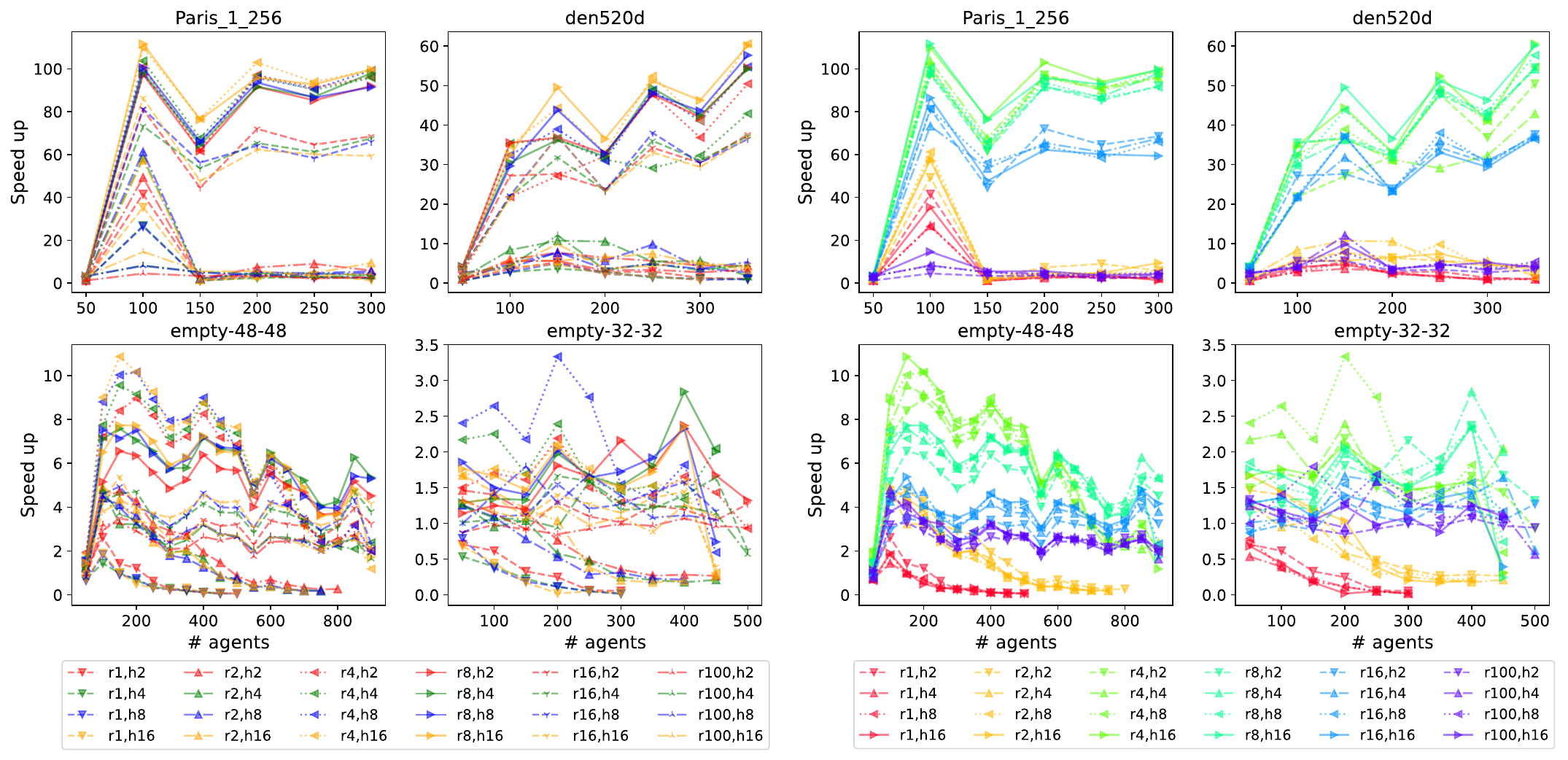}
    \caption{\textbf{WF-EECBS $w_h, r$ Analysis |} 
    % Left half: We expect performance to be primarily driven by $w_h$ with larger $w_h$ resulting in larger speed ups. Lines with the same color have the same $w_h$, and are expected to be grouped together with larger $w_h$ having higher speed ups. However we do not observe either of these, or any discernible pattern in general. \\
    % Right half: Lines with the same color now have the same $r$ values. We see a striking grouping effect, showing how performance is tightly linked to $r$ across different $r$ and $w_h$ values across all the maps. 
    Left half: Lines with the same color have the same $w_h$ values, each $w_h$ color had 6 different line/marker styles corresponding to different $r$ values. We expect performance to be primarily driven by $w_h$, with methods with same $w_h$ values performing similarly and larger $w_h$ resulting in larger speed ups. However we see that methods with the same $w_h$ value (e.g. all the blue lines) are wildly scattered. \\
    Right half: Lines with the same color have the same $r$ values. We see a striking grouping effect across runs with same $r$ values but different $w_h$ values, showing how performance is tightly linked to $r$ and \textit{not} $w_h$ across different $r$ and $w_h$ values across all the maps. 
    % Colors represent different ratios $r=w_c/w_h$ while markers and linestyle denote different $w_h$ values. One would initially think that $w_h$ would cause most of the speed-up while $w_c$ or $w_c/w_h$ plays a secondary effect. The colored bands show the opposite! Performance is tightly linked to $r$ across many $w_c$ and $w_h$ values across all maps. 
    % Interestingly, this relationship seems to "loosen" on empty-32-32, random-32-32-10, and random-32-32-20, the three maps with the smallest speed-ups.
    Additionally, values too low $r=1,2$ (red, yellow) and too high $r=16,100$ (blue, purple) perform worse than $r=4,8$ (lime green, turquoise), implying some optimal region of $r \in [2,16] $. 
    % We find in Figure \ref{fig:weightedFocal-regular-full} that $r=5$ works well across most maps.
    }
    \label{fig:weightedFocal-ratio-full}
\end{figure*}

% PP, ECBS, W-EECBS section
\begin{figure*}[p]
    \centering
    \includegraphics[width=0.93\textwidth]{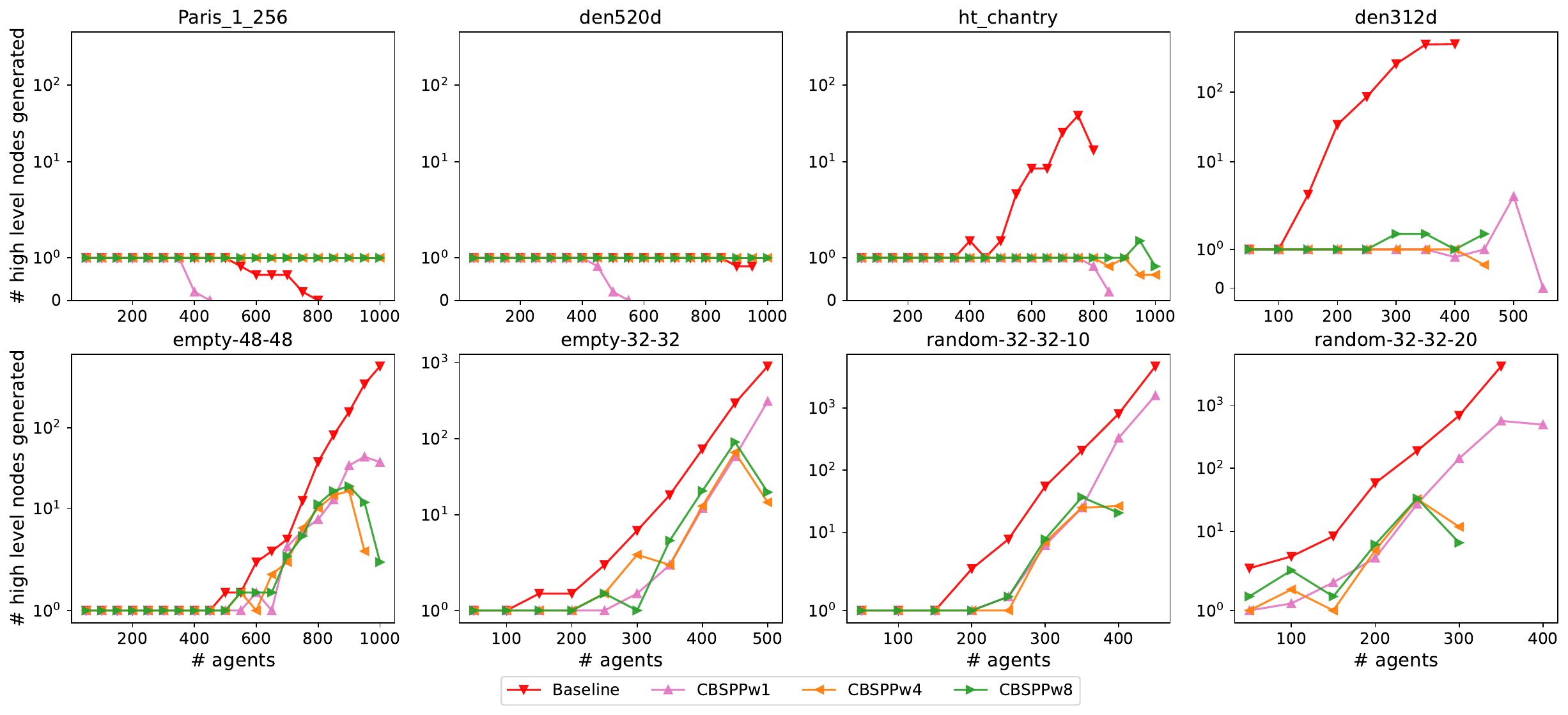}
    \caption{\textbf{PP, CBS, W-EECBS equivalence |} We see that running ``CBSPP" (WO-EECBS with a very large suboptimality factor simulating $w_{so} \gets \infty, w_h=w$, or equivalently WF-EECBS with $w_{so}, r \gets \infty, w_h = w$) is equivalent to running $w$ weighted prioritized planning with CBS's conflict resolution capability. This is highlighted by the number of high level nodes sticking to one with low numbers of agents (identical to PP) as opposed to the baseline with several high level nodes, and then increasing only after conflicts are forced. Observe how the larger maps (top row) are able to be solved in only one high level node (i.e. no conflict resolution required), but smaller maps require reasoning over conflicts. Note the differing log y-axis values across different graphs. Fractional values due to averaging across 5 seeds.}
    \label{fig:cbspp-node-full}
\end{figure*}

\subsection{Weighted Open}
Overall, performance with the weighted anchor variant is very varied based on the map; it provides large speed ups (10+) in 2, medium (1-5) in 3, and hurts (0-1) in 3. 
Figure \ref{fig:weightedAnchorLBImprovements} shows that improving the lower bound on the usage of CBS improvements does lead to higher utilization. Contrary to our expectations, Table \ref{tab:openWeightedSummary} and Figure \ref{fig:weightedAnchorResults} reveal this results in worse performance even though this computation has negligible overhead.
% ; understanding this phenomena is outside the scope of this work and is left as future work. 

The relative performance of $w_h$ in WO-EECBS fits our intuition with larger $w_h$ helping to a certain extent and then hurting due to the interplay with the focal queue. Concretely, WO-EECBS with a ``saturated" anchor weight of 2 provides lower speed-ups as the focal queue in that instance has $w_f = w_{so}/w_{so} = 1$ and thus flexibility to reduce the number of collisions.
The next section demonstrates that this variant is dominated by the weighted focal variant.

\begin{table}[htb]
\centering
% \caption{Weighted open Results}
\resizebox{0.45\textwidth}{!}{
% \caption{Table Title}
\begin{tabular}{|c|c||c|c|c|c|}
\hline
\multicolumn{2}{|c||}{Method} & \multicolumn{2}{|c|}{Speed up } & \multirow{2}{*}{\parbox{0.35\linewidth}{\centering $\%$ instances faster than Baseline}} & \multirow{2}{*}{\# solved} \\ \cline{1-4}
% $w_h$ & LB+ & Max & Median & \% \textgreater Baseline & \# solved \\ \hline
$w_h$ & LB+ & Max & Median & & \\ \hhline{|=|=#=|=|=|=|} 
1 & False & 1 & 1 & N/A, is baseline & 98 \\ \hline
1.2 & False & 3 & 1.4 & 69\% & 102 \\ %\hline
1.2 & True & 3 & 1.2 & 61\% & \textbf{103} \\ \hline
1.5 & False & \textbf{37} & \textbf{2.7} & 70\% & 101 \\ %\hline
1.5 & True & 36 & 2.4 & \textbf{72\%} & \textbf{103} \\ \hline
2 & False & 32 & 0.66 & 39\% & 67 \\ %\hline
2 & True & 28 & 0.31 & 14\% & 49 \\ \hline
\end{tabular}}
\caption{{\footnotesize
    \textbf{WO-EECBS Results |} 
We report the max and median speed up across all 8 maps, as well as the number of instances solved and better than the baseline.
We see that $w_h=1.5$ produces the best speed up and that in general improving the lower bound (LB+ set to True) decreases performance.}}
\label{tab:openWeightedSummary}
\vspace*{-2mm}
\end{table}

\begin{table}[h]
\centering
% \caption{Weighted Anchor Results}
\resizebox{0.45\textwidth}{!}{
% \caption{Table Title}
\begin{tabular}{|c|c||c|c|c|c|}
\hline
\multicolumn{2}{|c||}{Method} & \multicolumn{2}{|c|}{Speed up } & \multirow{2}{*}{\parbox{0.22\linewidth}{\centering $\%$ faster than Baseline}} & \multirow{2}{*}{\# solved} \\ \cline{1-4}
% $w_h$ & LB+ & Max & Median & \% \textgreater Baseline & \# solved \\ \hline
$r$ & $w_h$ & Max & Median & & \\ \hhline{|=|=#=|=|=|=|} 
$\infty$ & 1 & 1 & 1 & - & 98 \\ \hline
2.5 & 8 & 84 & 5.5 & 77\% & 106 \\ %\hline
2.5 & 16 & 112 & 6.8 & 79\% & 104 \\ \hline
5 & 4 & \textbf{131} & 9.5 & \textbf{92\%} & \textbf{113} \\ %\hline
5 & 8 & 127 & 9.4 & 89\% & 109 \\ %\hline
5 & 16 & 113 & \textbf{11} & 88\% & 109 \\ \hline
10 & 4 & 109 & 5.1 & 90\% & 111 \\ %\hline
10 & 8 & 91 & 7.2 & 89\% & 110 \\ \hline
\end{tabular}}
\caption{\footnotesize \textbf{WF-EECBS Results |} Comparing against Table \ref{tab:openWeightedSummary} we see that WF-EECBS greatly outperforms WO-EECBS and the baseline in the majority of instances. The first row describes the EECBS baseline in WF-EECBS parameters.}
\label{tab:focalSummary}
\vspace*{-2mm}
\end{table}

\subsection{Weighted Focal}
% Figure \ref{fig:weightedFocal-regular-full} show that like WO-EECBS, performance varies across different maps. 
Table \ref{tab:focalSummary} demonstrates that WF-EECBS's speed up is consistently higher than the baseline and WO-EECBS.
% The importance of $r$ builds intuition on why WF-EECBS is better; WO-EECBS's weighted anchor.
Overall WF-EECBS helps on 7 out of 8 maps, providing large speed ups (10+) on three and massive speed ups (50-100+) on two. Weighted EECBS (W-EECBS) therefore refers to this weighted focal version.

% Figure \ref{fig:weightedFocal-ratio-full} show a surprising and important relationship between the collision weight $w_c$ and cost-to-go weight $w_h$ in WF-EECBS; the performance is dominated by the ratio $r = w_c/w_h$ rather than the actual $w_c$ or $w_h$ weights, with optimal values $r \in [2,16]$. The ratio $r$ explicitly dictates the tradeoff between planning longer to avoid a future conflict (collision) or planning shorter and incurring the collision which will need to be resolved by the constraint tree afterwards. Regular EECBS lacks this flexibility and with $r \to \infty$ will prioritize planning longer to avoid conflicts. \textit{To highlight the importance of $r$, we reparameterize WF-EECBS in respect to $r$ and $w_h$ with $f_{focal}(g,h,c) = g+w_h*(h + r*c)$}. Table \ref{tab:focalSummary} shows that increasing $w_h$ with the same $r$ usually but not necessarily increases median speed up.
Figure \ref{fig:weightedFocal-ratio-full} show a surprising and important relationship between the cost-to-go weight $w_h$ and relative conflict weight $r$ in WF-EECBS; the performance is dominated by $r$ rather than $w_h$, with optimal values $r \in [2,16]$. The relative weight $r$ explicitly dictates the tradeoff between searching longer to avoid a future conflict or planning shorter and incurring the conflict which will need to be resolved by the constraint tree afterwards. Regular EECBS lacks this flexibility and with $r \to \infty$ will prioritize planning longer to avoid conflicts. Table \ref{tab:focalSummary} shows that increasing $w_h$ with the same $r$ usually but not necessarily increases median speed up.

% Figure \ref{fig:weightedFocal-regular-full} shows the affect of WF-EECBS across many maps; the use of $r$ helps more as the map size increases (starting on the bottom right, see y-axis changes as the maps get bigger). Additionally, for each of the smaller maps (bottom row), relative performance usually decreases as the number of agents increases. Both of these patterns fit our expectations; including $r$ allows the low-level planner to find paths faster in long maps rather  is more useful when paths are longer (larger maps) and less effective when there is more congestion (which would likely cause deviations from the heuristic).

\begin{table}[h]
\centering
% \caption{Weighted Anchor Results}
\resizebox{0.45\textwidth}{!}{
% \caption{Table Title}
\begin{tabular}{|c|c||c|c|c|c|}
\hline
\multicolumn{2}{|c||}{Method} & \multicolumn{2}{c|}{Speedup} & \multicolumn{1}{c|}{\multirow{2}{*}{\parbox{0.22\linewidth}{\centering $\%$ faster than Baseline}}} & \multicolumn{1}{c|}{\multirow{2}{*}{\parbox{0.2\linewidth}{\centering \# solved vs Baseline}}} \\ \cline{1-4}
-CBS & \multicolumn{1}{c||}{$w_{so}$} & \multicolumn{1}{c|}{Max} & \multicolumn{1}{c|}{Median} &  &  \\ \hhline{|=|=#=|=|=|=|} 
% \multirow{2}{*}{Method} & \multicolumn{1}{c}{\multirow{2}{*}{$w_{so}$}} & \multicolumn{2}{c}{Speedup} & \multicolumn{1}{c}{\multirow{2}{*}{baseline}} & \multicolumn{1}{c}{\multirow{2}{*}{\#solved}} \\
%  & \multicolumn{1}{c}{} & \multicolumn{1}{c}{Max} & \multicolumn{1}{c}{Median} & \multicolumn{1}{c}{} & \multicolumn{1}{c}{} \\
E & 1.01 & 1.7 & 0.45 & 23\% & 30/42 \\
EE & 1.01 & 1.8 & 0.71 & 32\% & 36/46 \\ \hline
E & 1.1 & 9.9 & 2.6 & 75\% & 71/76 \\
EE & 1.1 & 10 & 2.3 & 73\% & 68/75 \\ \hline
E & 1.2 & 16 & 3.7 & 80\% & 83/77 \\
EE & 1.2 & 22 & 3.1 & 80\% & 79/77 \\ \hline
E & 1.5 & 35 & 3.5 & 83\% & 95/73 \\
EE & 1.5 & 47 & 3.7 & 78\% & 94/73 \\ \hline
E & 2 & 88 & 5.5 & 85\% & 108/77 \\
EE & 2 & 88 & 5.5 & 85\% & 105/77 \\ \hline
E & 4 & 137 & 7.8 & 91\% & 110/66 \\
EE & 4 & 137 & 8.5 & 92\% & 110/67 \\ \hline
E & 8 & 164 & 6.9 & 91\% & 112/69 \\
EE & 8 & 164 & 9.6 & 91\% & 111/68 \\ \hline
\end{tabular}}
\caption{\footnotesize \textbf{Generalizing weighting FOCAL to different suboptimal CBS methods and suboptimalities. |} We compare the effect of weighting FOCAL on both ECBS and EECBS across different suboptimalities. We use $r=5, h=8$ and a timeout of 60 seconds across all experiments, and report statistics as in Table \ref{tab:openWeightedSummary}. The last column shows the number of instances solved (numerator) vs the baseline (denominator). We see that incorporating the weighted FOCAL hurts at a very low suboptimality $w_{so}=1.01$, but then produces large benefits for $w_{so} \geq 1.5$. Additionally, we see very similar speedups across different methods at the same suboptimality, demonstrating how our method's benefits are generalizable across different suboptimal CBS methods.}
\label{tab:suboptimality-joined-summary}
\end{table}

We check how incorporating the weighted heuristic generalizes across different methods and suboptimalities. Table \ref{tab:suboptimality-joined-summary} shows the effect of using a weighted focal with $r=5, h=8$ across different suboptimalities on ECBS and EECBS. These hyper-parameters were chosen solely based on Table \ref{tab:focalSummary} (on EECBS with $w_{so}=2$) and were intentionally not optimized for this experiment.
% to see if our weighted heuristic hyper-parameters generalized well. 
We employ a timeout of 60 seconds and compare against the corresponding unweighted baseline. We observe that our weighted focal hurts at a very low suboptimality $w_{so}=1.01$ but then steadily results in larger performance boosts as $w_{so}$ increases. In particular, for large suboptimalities $w_{so} \geq 1.5$, we see that our weighted methods start to solve significantly more instances than the baseline. For $w_{so} \geq 2$, we also get large and consistent speed up benefits ($> 80\%$ faster than baseline, median speed up $> 5$) and with $w_{so} = 4, 8$ solve almost double the number of instances as their unweighted baseline. It is important to observe that our method produces very similar speeds up in both ECBS and EECBS.
% , highlighting how our technique identically speeds up the low level planner regardless of the high level search, demonstrating how our technique is readily generalizable to other existing and future sub-optimal CBS methods.
This highlights how our technique identically speeds up the low level planner regardless of the high level search, demonstrating how our technique is readily generalizable to future suboptimal CBS methods. 

% \begin{figure}[t]
%     \centering
%     % \includesvg[width=.47\textwidth]{figures/WFExplanations.svg}
%     \includegraphics[width=2.5in]{figures/weecbs_new_figure.eps}
%     \caption{\footnotesize We show four possible paths for the low level planner to choose from and the corresponding $r$ values that would cause them to be picked in WF Suboptimal CBS. There are two main explanations on why weighting FOCAL performs better than (unweighted) normal FOCAL. First, choosing longer paths with few current conflicts could result in more future conflicts. Second, finding these paths make take too long and it is better to find only slightly longer paths with fewer conflicts, at the expense of more high level CT nodes later. Our experimental analysis provides evidence supporting the second hypothesis but not the first.}
%     \label{fig:illustrative-example}
% \end{figure}

\subsection{Understanding Weighted Suboptimal CBS}
Figure \ref{fig:weightedFocal-ratio-full} surprised us as it contradicted our initial intuition from single agent planning that the cost-to-go weight $w_h$ would have a direct effect on performance, with larger $w_h$ generally causing larger speed ups. Instead, the relative conflict weight $r$ primarily dictates performance, with increasing $w_h$ only marginally increasing performance given $r$. 
Table \ref{tab:understanding} shows how $r$ directly controls the balance between low level and high level work in WF-EECBS. Decreasing $r$ from $\infty$ (which is what EECBS implicitly has) shifts the work load from the low level to the high level search.
% WF-EECBS consistently generates more CT nodes, but overall generates substantially fewer low-level nodes by having computing each CT node significantly faster. 
Across all 8 maps, we generally observe that each WF-EECBS's low level planner call does 5-150x fewer low level expansions than EECBS while WF-EECBS's CT search does 2-20x more work than EECBS, resulting in a net reduction of 5-100x less total low-level nodes expanded. The increased numbers of CT nodes does add high level overhead, resulting in a sweet spot in the middle.
This leads us to believe that current Suboptimal CBS methods place too large a burden on the low level search; each CT node requires too many expansions to find a suboptimal path with minimal conflicts. Our WF method is able to ``balance" the low level and high level work better, and improves performance by expanding more CT nodes significantly faster than their unweighted counterpart.
Figure \ref{fig:illustrative-example} illustrates this effect of using our Weighted Focal method compared to CBS and current Suboptimal CBS methods.

% The role of $r$ might be obvious in retrospect but the impact leads to a novel insight; minimizing conflicts as a separate mechanism loosely connected to path length (e.g. FOCAL) can lead to the low-level planner doing significantly too much work.
The role of $r$ might be obvious in retrospect but the impact leads to a novel insight; minimizing conflicts as a separate mechanism loosely connected to path length (as currently done in FOCAL) can lead to the low-level planner doing significantly too much work.
We believe this finding has direct relevance for future Suboptimal CBS based methods, and even non-CBS MAPF methods like MAPF-LNS2 \citep{li2022mapf-lns2} whose low level planner minimizing unweighted conflicts will likely suffer similarly. We predict that using an explicit trade off with $r$ will better balance low-level and high-level work and improve performance.

\begin{table}[tb]
\centering
\resizebox{0.49\textwidth}{!}{
\begin{tabular}{|c|c||c|c|c|c|} \hline
\multicolumn{2}{|c|}{Method} & \multirow{2}{*}{\parbox{0.22\linewidth}{\centering Total \# Low Level Nodes}} & \multirow{2}{*}{\parbox{0.22\linewidth}{\centering \# CT Nodes + bypasses}} & \multirow{2}{*}{\parbox{0.22\linewidth}{\centering \# LL per low level call}} & \multirow{2}{*}{Speed up} \\ \cline{1-2}
$r$ & $w_h$ &  &  &  &  \\ \hhline{|=|=#=|=|=|=|}
$\infty$ & $-$ & 1,046,159 & 20.2 & 6,145 & 1 \\ \hline
16 & 4  & 223,000 & 22.8 & 1,294 & 3.9 \\
8 & 4 & 135,196 & 37.8 & 702 & 6.1 \\
4 & 4 & 51,000 & 69.4 & 228 & 9.9 \\
2 & 4 & 56,224 & 312 & 121 & 1.8 \\ \hline
\end{tabular}}
\caption{{\footnotesize \textbf{Comparing low and high level statistics |} WF-EECBS expands less total low level nodes by generating more CT nodes significantly faster (fewer low level nodes per low level call) than EECBS (top row). Although we show one example ($w_{so}=2$, den312d with 150 agents), we observe the relative conflict weight $r$ controls this balance throughout different $w_{so}$, $w_h$, and scenarios.}}
\label{tab:understanding}
\vspace*{-1em}
\end{table}

% \begin{table}[tb]
% \centering
% \resizebox{0.49\textwidth}{!}{
% \begin{tabular}{|c|c||c|c|c|c|} \hline
% \multicolumn{2}{|c|}{Method} & \multirow{2}{*}{\parbox{0.22\linewidth}{\centering \# LL per low level call}} & \multirow{2}{*}{\parbox{0.22\linewidth}{\centering \# High Level Nodes}} & \multirow{2}{*}{\parbox{0.22\linewidth}{\centering Total \# Low Level Nodes}} & \multirow{2}{*}{Speed up} \\ \cline{1-2}
% $r$ & $w_h$ &  &  &  &  \\ \hhline{|=|=#=|=|=|=|}
% $\infty$ & $-$ & 6,145 & 20.2 & 1,046,159 & 1\\ \hline
% 16 & 1 & 4886 & 22.8 & 223,000 & 3.9 \\
% 8 & 1 & 2909 & 37.8 & 135,196 & 6.1 \\
% 4 & 1 & 1308 & 69.4 & 51,000 & 9.9 \\
% 2 & 1 & 1003 & 312 & 56,224 & 1.8 \\ \hline
% \end{tabular}}
% \end{table}

% \begin{table}[tb]
% \centering
% \resizebox{0.49\textwidth}{!}{
% \begin{tabular}{|c|c||c|c|c|c|} \hline
% \multicolumn{2}{|c|}{Method} & \multirow{2}{*}{\parbox{0.22\linewidth}{\centering \# LL per low level call}} & \multirow{2}{*}{\parbox{0.22\linewidth}{\centering \# High Level Nodes}} & \multirow{2}{*}{\parbox{0.22\linewidth}{\centering Total \# Low Level Nodes}} & \multirow{2}{*}{Speed up} \\ \cline{1-2}
% $r$ & $w_h$ &  &  &  &  \\ \hhline{|=|=#=|=|=|=|}
% $\infty$ & $-$ & 6,145 & 20.2 & 1,046,159 & 1\\ \hline
% 16 & 1 & 4,886 & 54 & 1,018,768 & 0.55 \\
% 8 & 1 & 2,909 & 53 & 588,922 & 1.00 \\
% 4 & 1 & 1,874 & 77.2 & 418,764 & 1.40 \\
% 2 & 1 & 1,308 & 117.6 & 348,072 & 1.58 \\ \hline
% \end{tabular}}
% \end{table}

\subsection{Relating CBS, Prioritized Planning, and W-EECBS}
We run WO-EECBS with a very large sub-optimality value ($w_{so}=10000$) and different anchor weights to see how this mimics running weighted prioritized planning. We denote these as ``CBSPP" with their specific weights to emphasize the relation. Figure \ref{fig:cbspp-node-full} verifies that the number of generated CT nodes stays at 1 for low levels of agents until conflicts become unavoidable. 
% Figure \ref{fig:cbspp-success-full} demonstrates how CBSPP's ability to replan using CBS's conflict resolution increases success rate compared to prioritized planning.
The bottom row plots with CT nodes greater than 1 also demonstrate how CBSPP's ability to replan using CBS's conflict resolution increases success rate compared to prioritized planning which would fail in those instances. 
At the expense of additional engineering complexity, we recommend that practitioners using PP should instead use W-EECBS with a large suboptimality as they get the same prioritized planning behavior in the root node along with the natural robustness and completeness of CBS.

\section*{Future Work and Conclusion}
We see several avenues to directly build upon our work.
% The varying performance across different maps is ripe for future work that takes in MAPF instances and Predicting the optimal $r$ and $w_h$ values on different maps could improve the current varying performance.
Our work keeps $r$ and $w_h$ fixed in MAPF instances; adaptively changing $r$ and $w_h$ during a single MAPF search, or predicting a fixed optimal $r$ and $w_h$ could increase performance and robustness across different maps. 
Determining the reason behind WO-EECBS improved bound's negative performance effect would also be interesting investigative work.

Our experiments provide compelling evidence for MAPF practitioners to use Weighted EECBS and more broadly incorporate relative conflict weights along with cost-to-go heuristics. We first introduce WO-EECBS which incorporating the weighted cost-to-go in the open queue, and analyze the effect of improving the lower bound on utilizing prioritized conflicts and symmetry reasoning. We then introduce WF-EECBS by modifying the focal priority to include a weighted cost-to-go and relative weighted conflict heuristic, and show significant speeds up compared to EECBS. We demonstrate how these speeds ups change across different hyper-parameters $w_{so}, w_h, r$ and different scenario (map sizes, numbers of agents). We provide novel insight that the cost-to-go weight $w_h$ does not primarily impact performance as expected, but that instead the relative weight $r$ dictates performance by trading off low-level and high-level work effectively. We show that our weighted focal technique results in similar speed ups regardless of the high level search, illustrating how our technique is readily generalizable to other suboptimal CBS methods. Finally, we show that PP is actually just one specific step in suboptimal CBS with an infinite sub-optimality, and show Weighted Suboptimal CBS is the natural generalization of the two. 

Overall, our proposed methods bear no additional overhead and are directly usable in other suboptimal CBS planners. More broadly, we hope this work inspires future MAPF work to incorporate the conflict heuristic with more nuance, and shows how a weighted cost-to-go heuristic can be successfully incorporated in CBS based methods.

% \textbf{Acknowledgement.} This material is partially supported by the National Science Foundation Graduate Research Fellowship under Grant No. DGE1745016 and DGE2140739.

% \clearpage
\bibliography{ref} 

\clearpage
\appendix

\setcounter{figure}{0}
\renewcommand{\thefigure}{A\arabic{figure}}
\setcounter{table}{0}
\renewcommand{\thetable}{A\arabic{table}}

\section{Quick Recap}
% We provide a quick recap for the skimming reader.

\subsection{Recommended background reading}
Before reading this paper, readers new to focal search are recommended to read \citet{anytimefocalsearch}, and readers new to bounded sub-optimal CBS are recommended to read ECBS \citep{barer2014suboptimal}. Readers interested in our lower bound improvements and how they relate to bounded suboptimal CBS should read \citet{improvedLowerBound} and EECBS \citep{li2021eecbs}.

\subsection{Intended takeaways}
\textbf{Main takeaways:} \par
% \subsubsection{Main takeaways:} \\
1. We show two ways of incorporating a weighted cost-to-go heuristic in bounded suboptimal CBS (e.g. ECBS, EECBS) and show that these can be used effectively contrary to existing MAPF intuition. \par
2. We find that weighting both the cost-to-go and conflict heuristic can obtain large speed ups by better balancing low-level and high level work via changing the relative importance of finding a solution fast or avoiding conflicts in our low-level focal search. Specifically, FOCAL should be changed from sorted by just $c$ to sorted by $g+w_h*(h+r*c)$ where $c$ is the number of conflicts on the current path, $w_h$ is the weighted cost-to-go hyper-parameter and $r$ is a relative conflict heuristic hyper-parameter. We show that contrary to single agent planning intuition, increasing $w_h$ does \textit{not} primarily improve performance, but that instead $r$ strongly dictates performance with $w_h$ playing a secondary role. We believe this subtly of incorporating the conflict heuristic has been overlooked in existing MAPF work.
% We can effectively incorporate a weighted cost-go-heuristic in bounded suboptimal CBS (e.g. ECBS, EECBS) by modifying FOCAL to rather than modifying OPEN. Specifically, FOCAL should be changed from sorted by just $c$ to sorted by $g+w_h*h+w_c*c$ where $c$ is the number of conflicts on the current path. Additionally, the ratio of $r=w_c/w_h$ primarily determines performance rather than $w_c$ or $w_h$, resulting in reparameterizing $r$ and $w_h$ with $f_{focal}(g,h,c) = g+w_h*(h + r*c)$. 
Our change can result in large (50+ speed ups) and solve significantly more problem instances. Our performance gain generalizes to both ECBS and EECBS, suggesting it can be directly helpful in other suboptimal CBS methods. \par 
3. Lastly, we show how (weighted) PP is a sub-step of (weighted) suboptimal CBS and how our method relates the two.
% \textbf{Weighted Open Variant (WO-EECBS):} 

\subsubsection{Weighted Open Variant (WO-EECBS):}
We can incorporate a weighted cost-go-heuristic in the open list (OPEN) and keep the focal list (FOCAL) with the conflict heuristic un-changed. Doing so limits $w_h$ by $w_{so}$ as well as reduces the flexibility/effectiveness of FOCAL to maintain bounded sub-optimality. In experimental results, WO-EECBS does not produce much consistent speed-up. 
Improving the lower bound increases utilization of prioritized conflicts and symmetry reasoning, but actually hurts runtime performance.
% \textbf{Weighted Focal Variant (WF-EECBS):}

\subsubsection{Weighted Focal Variant (WF-EECBS):}
We keep OPEN unweighted and instead incorporate the weighted heuristic in FOCAL along with the inadmissible conflict heuristic via $g+w_h*(h+r*c)$. We see that performance is driven by $r$ which determines the relative importance of finding a solution fast (lower $r$) vs avoiding conflicts (higher $r$), allowing us to explicitly reason between the two. We show that our use of $r$ changes the low and high level work, and that WF-EECBS $w_{so}=2$ with our roughly optimal $r=5$ results in 5-150x less low level work per CT node and 2-20x more CT nodes than EECBS, resulting in a net reduction of 5-100x less work in total. \par
Overall WF-EECBS helps on 7 out of 8 maps, providing large speed ups (10+) on three and massive speed ups (50-100+) on two compared to EECBS. Weighted EECBS (W-EECBS) therefore refers to this weighted focal version.
% Weighting the cost-go-heuristic $w_h$ helps more as the map size increases. Additionally, for each of the smaller maps, relative performance usually decreases as the number of agents increases. Both of these patterns fit our intuition; weighting the cost-to-go heuristic is more useful when paths are longer (larger maps) and less effective when there is more congestion (smaller maps with more agents) which would likely cause deviations from the heuristic.
% \textbf{Relating CBS, Prioritized Planning, and W-EECBS:} \par

\subsubsection{Relating CBS, Prioritized Planning, and W-EECBS:}
We prove that PP is actually equivalent to generating the initial agent paths in the root CT node in bounded sub-optimal CBS planners like (ECBS, EECBS) with an infinite sub-optimality and prove that W-EECBS is the naturalize generalization of weighted PP and EECBS. We shorthand our W-EECBS method with $r \gets \infty$ as CBSPP, and experimentally verify how the number of generated nodes stays at 1 for low levels of agents until conflicts become unavoidable. We demonstrate how CBSPP's ability to replan using CBS's conflict resolution increases success rate compared to prioritized planning. At the expense of additional engineering complexity, practitioners using PP should try using W-EECBS with a large sub-optimality as they get the same prioritized planning behavior in the root node along with the natural robustness and completeness of CBS.

\section{Justifying removing timeouts from plots}
\label{section:appendix-justify-timeouts}
In all figures, if a method fails (times out on all 5 seeds) on a particular number of agents on a map, we do not report larger number of agents as this causes misleading visuals. A reminder that the speed up $S_{method} = T_{baseline}/T_{method}$ (larger is better) is reported to normalize differences in hardware, where the baseline is the unweighted method (ECBS or EECBS based on context). Figure \ref{fig:figure-explanation} demonstrates an example where including timeouts causes different methods to appear to have the same result, as well as causes false trends on their behaviour compared to the baseline.

\begin{figure}[h]
    \centering
    \includegraphics[width=0.45\textwidth]{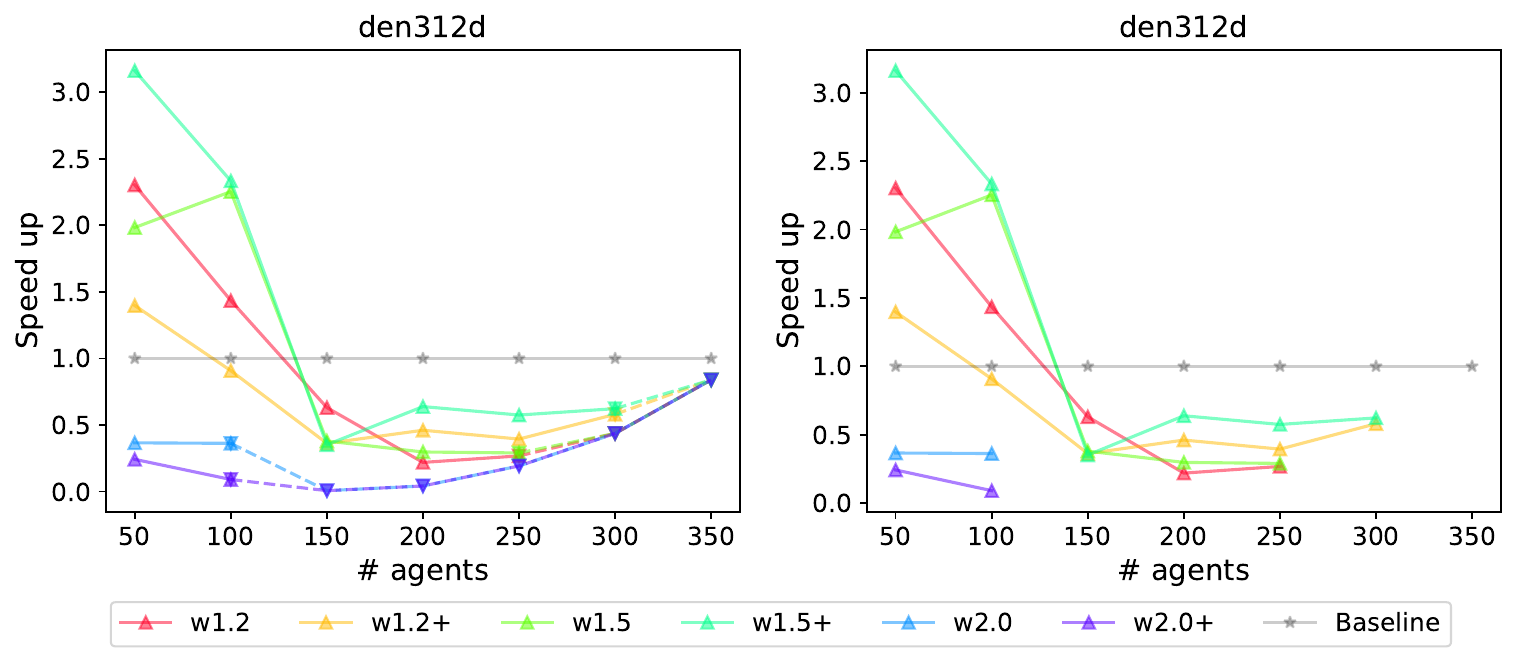}
    \caption{\footnotesize \textbf{Justifying removing timeouts |} The left image shows the raw speed up including instances which have timed out (in downward triangle marker and dashed lines) and instances which finished within the timeout $T_{max}$ (in upward triangle marker, solid lines, and starting at the leftmost of each plot). We see that the timed out instances all have the same values at the bottom as the speed up $S_{method}=T_{baseline}/T_{method}=T_{baseline}/T_{max}$. This causes the false impression that different failed methods have the same speed up, and that the speed up increases as the number of agents increases (which is actually caused by $T_{baseline}$ increasing). The right subplot without these failed instances displays the results much more accurately.}
    \label{fig:figure-explanation}
    \vspace*{-2mm}
\end{figure}

\section{Additional Plots}
We provide additional plots that showcase how performance changes over different maps.
\begin{figure*}[p]
    \centering
    \includegraphics[width=1\textwidth]{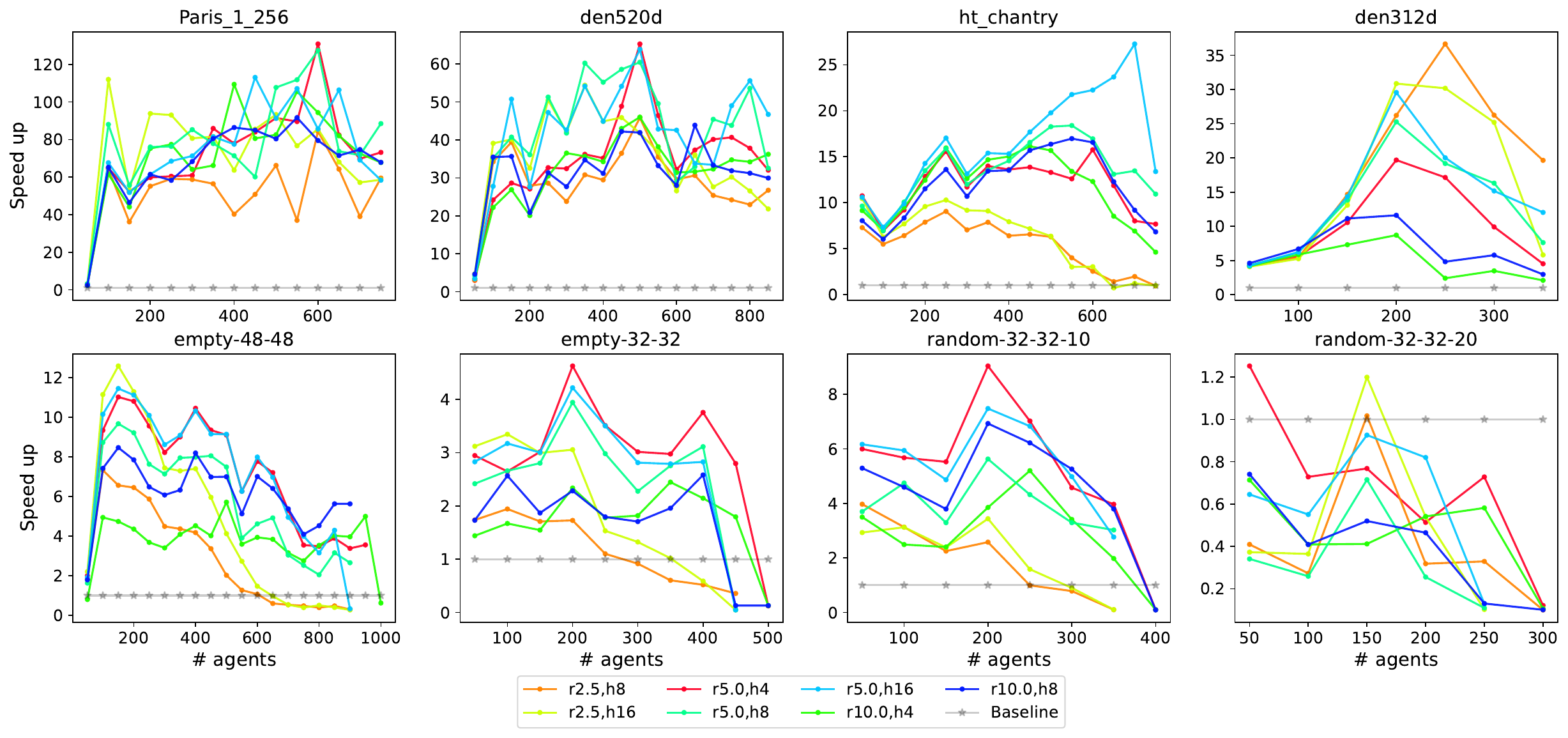}
    \caption{\textbf{WF-EECBS Results |} WF-EECBS produces a speed-up factor of 10 or higher on half the maps, and a smaller speed up on three, while performing worse than the baseline on just the random-32-32-20 map. Note the changes in y-axis.}
    \label{fig:weightedFocal-regular-full}
\end{figure*}

\begin{figure*}[p]
    \centering
    \includegraphics[width=1\textwidth]{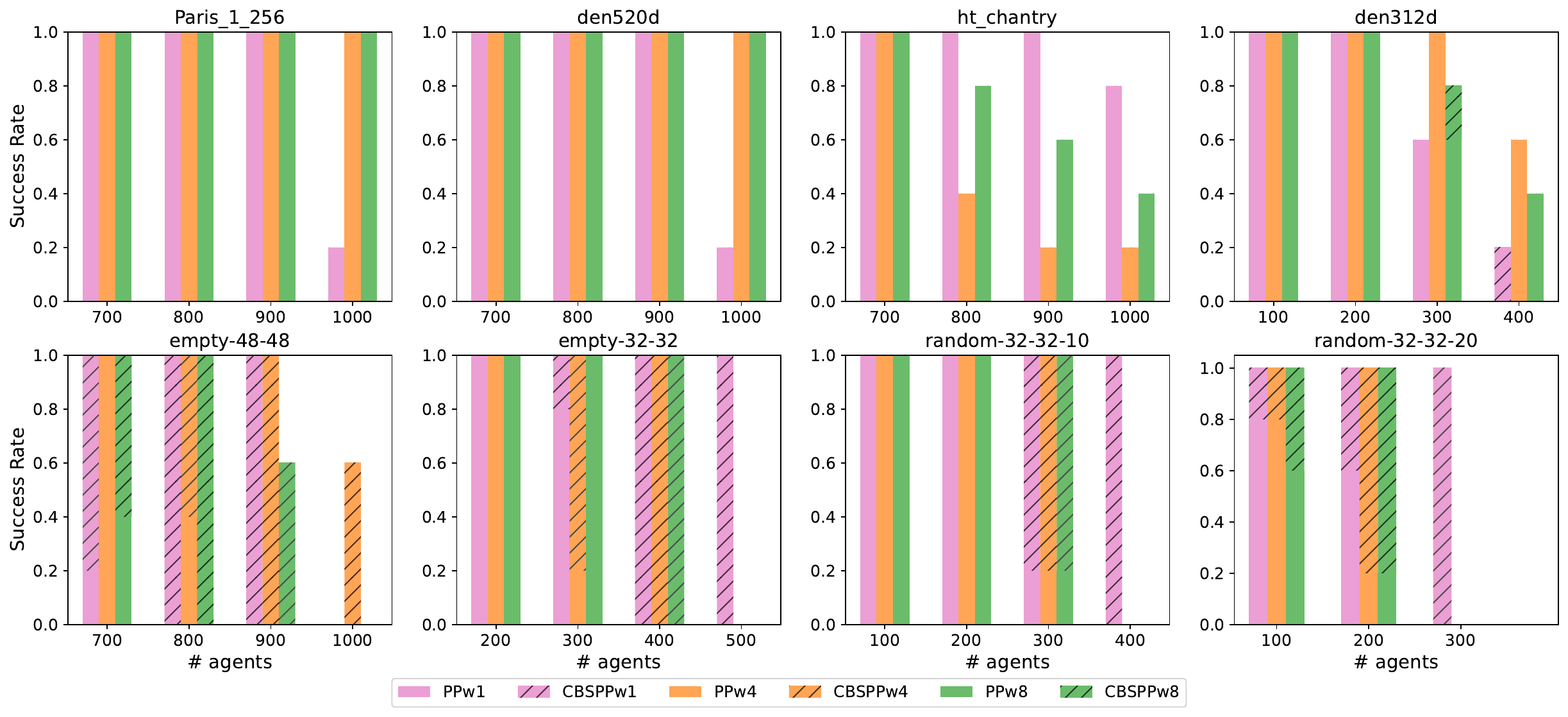}
    \caption{\textbf{PP vs CBSPP Success rate |} The increased success rate of the shaded region (CBSPP) over solid (PP) across different weights show the benefit of using EECBS's conflict resolution in high agent regimes or small maps where conflicts become unavoidable. Again observe how the larger maps (top row) are able to be solved with PP (i.e. no conflict resolution required), but smaller maps require reasoning over conflicts.}
    \label{fig:cbspp-success-full}
\end{figure*}

\section{Update March 2024: Speed-ups Dependent on Conflict Heuristic}
We discovered that regular EECBS's performance depends heavily on how the conflict heuristic is computed in the low-level search. In particular, the standard low-level search computes conflicts in respect to vertex and edge collisions with other agents' paths. Our codebase had an additional optimization which incorporated ``target" conflicts where an agent rests on its goal and a later agent crosses over it. This optimization, although resulting in a more accurate estimate of conflicts, hurts performance.

Concretely, given that an agent $a_i$ reached its goal state $s_g$ at timestep $t_1$, we look at all other agents $a_j$ and checked if they traversed $s_g$ at a timestep $t_2 > t_1$. If so we added this to the conflict heuristic. It ends up that including this significantly hurts the baseline performance. In the low level search, when the agent generates the goal node, the goal node could incur this penalty to have a high conflict penalty, and thus does not get expanded from the FOCAL list for a long until it finds a path to the goal with $t' > t_2$ or exhausts all $w_{so}$ suboptimal paths with lower conflicts. Without this computation, the agent generates and expands the goal node to quickly find a path that incurs a goal conflict, but this can be more efficiently resolved via constraints.

When removing this optimization, the baseline performed significantly better, and the relative improvement of our method substantially decreased. In particular, the current results show how W-EECBS helps in large maps Paris\_1\_256 and den520d. This is because the ``target" conflict computation causes low-level agents in EECBS to search over a large space before it can expand the goal node with the accurate-but-larger conflict value. With the lower inaccurate conflict estimate, the EECBS low-level search can immediately expand the goal node without additional search effort. 

We thus reran the main experiments to see how performance varies across different weights $r,w_h$ and $w_{so}$ suboptimalities. All experiments tested on the same 8 maps with the same settings except with a 60 second timeout and agent stepsize of 100 (instead of 50).
Compare Figure \ref{fig:updated-weights} with Figure \ref{fig:weightedFocal-ratio-full}, Figure \ref{fig:updated-suboptimality} with Figure \ref{fig:weightedFocal-regular-full}, Table \ref{tab:updated-suboptimality} with Table \ref{tab:suboptimality-joined-summary}, and Table \ref{tab:updated-understanding} with Table \ref{tab:understanding}. Note the larger agent stepsize means that there were roughyl 1/2 as many possible problems in the last column of Table \ref{tab:updated-suboptimality} than in Table \ref{tab:suboptimality-joined-summary}.

Figure \ref{fig:updated-weights} shows how overall performance changes substantially as our previous two strongest maps (Paris\_1\_256 and den520d) now only have marginal performance benefits. We see that performance trends as the number of agents increases differs between maps. Figure \ref{fig:updated-suboptimality} shows that the $w_{so}$ suboptimality also substantially effects performance similar to before. Note that the median speedup looks lown in Table \ref{tab:updated-suboptimality} as we compute it across solved instances in all maps and the large maps (with smaller speed-ups) have more such instances (as they allow more agents). Practitioners should therefore experiment with $r$ to see what value works best for their problem instances. Apart from the large maps, the original findings are followed except scaled down. Table \ref{tab:updated-understanding} shows a similar pattern as discussed in Table \ref{tab:understanding}, highlighting that the $r$'s role is robust to different calculations of the conflict heuristic.

% (e.g. updating Figure \ref{fig:weightedFocal-ratio-full} and Table \ref{tab:understanding}) as well across different $w_{so}$ suboptimalities (e.g. updating Table \ref{tab:suboptimality-joined-summary}).

% \columnbreak
% \clearpage
% \FloatBarrier
\newpage

\begin{table}[h]
\centering
% \caption{Weighted Anchor Results}
\resizebox{0.45\textwidth}{!}{
% \caption{Table Title}
\begin{tabular}{|c|c||c|c|c|c|}
\hline
\multicolumn{2}{|c||}{Method} & \multicolumn{2}{c|}{Speedup} & \multicolumn{1}{c|}{\multirow{2}{*}{\parbox{0.22\linewidth}{\centering $\%$ faster than Baseline}}} & \multicolumn{1}{c|}{\multirow{2}{*}{\parbox{0.2\linewidth}{\centering \# solved vs Baseline}}} \\ \cline{1-4}
-CBS & \multicolumn{1}{c||}{$w_{so}$} & \multicolumn{1}{c|}{Max} & \multicolumn{1}{c|}{Median} &  &  \\ \hhline{|=|=#=|=|=|=|} 
EE & 1.01 & 0.97 & 0.64 & 14\% & 16/19 \\ \hline
EE & 1.1 & 1.61 & 1.05 & 61\% & 34/36 \\ \hline
EE & 1.2 & 2.78 & 1.10 & 66\% & 37/41 \\ \hline
EE & 1.5 & 7.45 & 1.30 & 77\% & 45/47 \\ \hline
EE & 2 & 16.09 & 1.46 & 96\% & 52/53 \\ \hline
EE & 4 & 26.93 & 1.44 & 86\% & 57/54 \\ \hline
EE & 8 & 26.28 & 1.50 & 94\% & 56/52 \\ \hline
\end{tabular}}
\caption{Analogous to Table \ref{tab:suboptimality-joined-summary}, we use $r=5, h=8$ and a timeout of 60 seconds with the updated conflict heuristic computation. The last column shows the number of instances solved (numerator) vs the baseline (denominator). We see that incorporating the weighted FOCAL hurts at a very low suboptimality $w_{so}=1.01$, but then produces benefits for $w_{so} \geq 1.5$. Although the median speed-up seems very low, Figure \ref{fig:updated-suboptimality} shows that this is very map and agent dependent, with W-EECBS providing benefits for $>$2x speed-ups for more than half the maps (these maps just have fewer runs than the larger maps which have more possible agents).}
\label{tab:updated-suboptimality}
\end{table}

\begin{table}[h]
\centering
\resizebox{0.49\textwidth}{!}{
\begin{tabular}{|c|c||c|c|c|c|} \hline
\multicolumn{2}{|c|}{Method} & \multirow{2}{*}{\parbox{0.22\linewidth}{\centering Total \# Low Level Nodes}} & \multirow{2}{*}{\parbox{0.22\linewidth}{\centering \# CT Nodes + bypasses}} & \multirow{2}{*}{\parbox{0.22\linewidth}{\centering \# LL per low level call}} & \multirow{2}{*}{Speed up} \\ \cline{1-2}
$r$ & $w_h$ &  &  &  &  \\ \hhline{|=|=#=|=|=|=|}
$\infty$ & $-$ & 3,583,976 & 56.1 & 13,487 & 1 \\ \hline
16 & 4 & 1,090,008 & 62.6 & 3,876 & 4.89 \\
8 & 4 & 379,080 & 64.7 & 1,392 & 9.32 \\
4 & 4 & 184,035 & 124 & 555 & 15.07 \\
2 & 4 & 272,493 & 712 & 301 & 3.15 \\ \hline
\end{tabular}}
\caption{{\footnotesize \textbf{Comparing low and high level statistics |} WF-EECBS expands less total low level nodes by generating more CT nodes significantly faster (fewer low level nodes per low level call) than EECBS (top row). Although we show one example ($w_{so}=2$, den312d with 200 agents), we observe the relative conflict weight $r$ controls this balance throughout different $w_{so}$, $w_h$, and scenarios.}}
\label{tab:updated-understanding}
\vspace*{-1em}
\end{table}

% \balance

\begin{figure*}[t]
    \centering
    \includegraphics[width=1\textwidth]{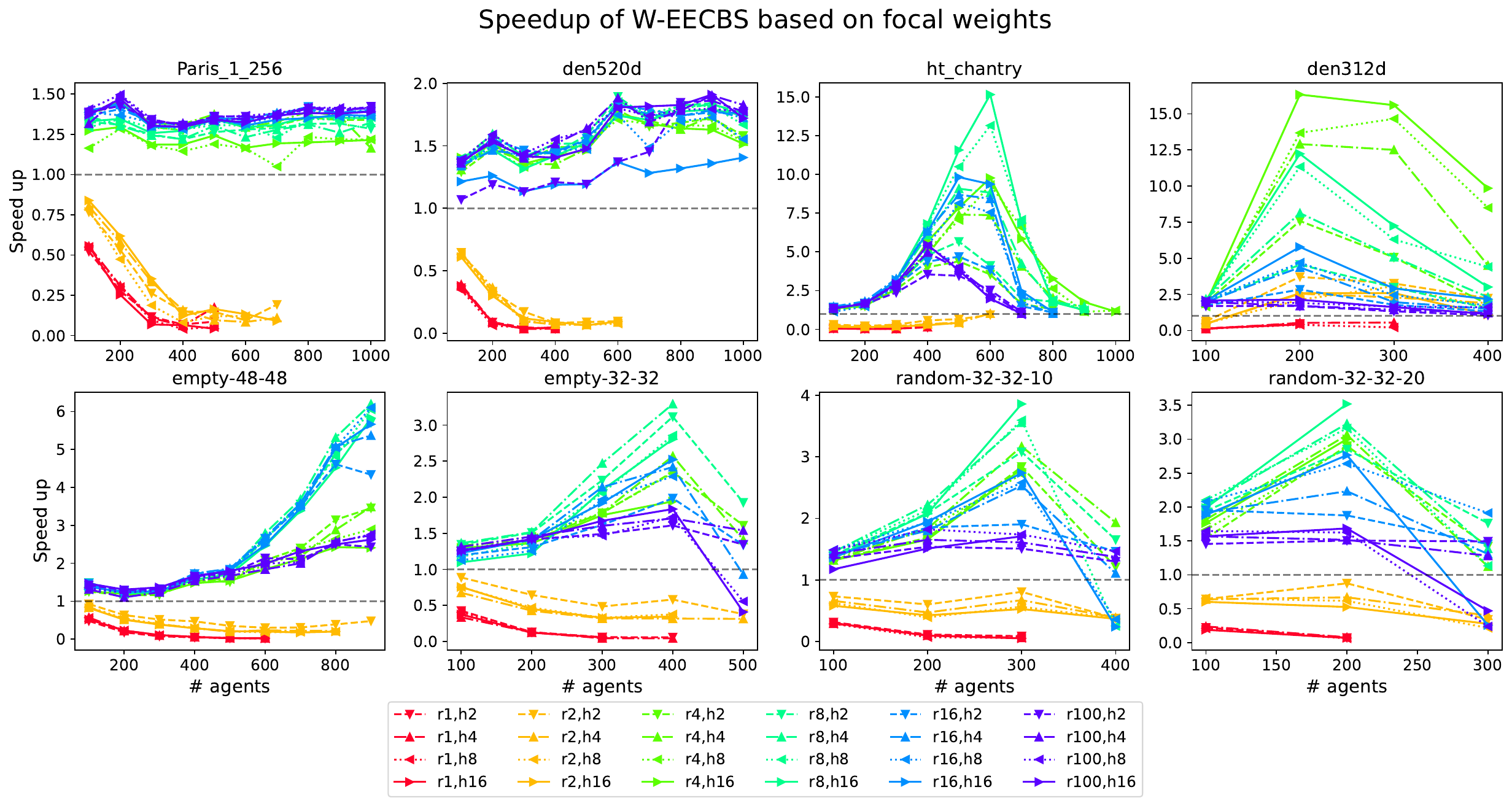}
    \caption{\textbf{W-EECBS Results with $w_{so}=2$ |} We notice that, like before, performance is dependent on $r$. However the magnitudes are substantially different compared to Figure \ref{fig:weightedFocal-ratio-full} or \ref{fig:weightedFocal-regular-full} for the two largest maps (note the differences in y-axis).}
    \label{fig:updated-weights}
\end{figure*}

\begin{figure*}[t]
    \centering
    \includegraphics[width=1\textwidth]{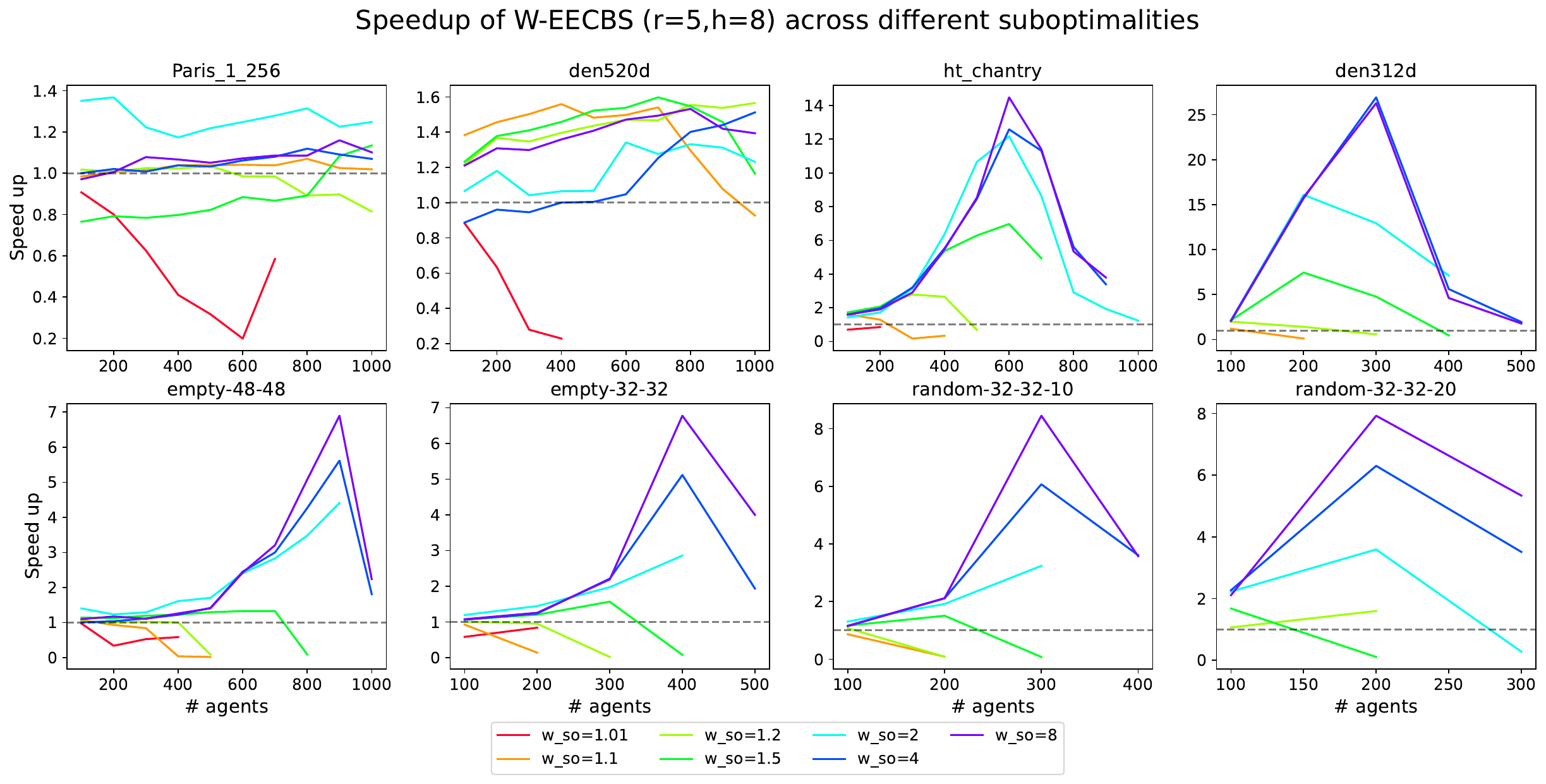}
    \caption{\textbf{W-EECBS Results with different suboptimality $w_{so}$ |} Similar to before, W-EECBS's relative performance to EECBS is dependent on the $w_{so}$. As $w_{so}$ increases, W-EECBS provides larger speedups.}
    \label{fig:updated-suboptimality}
\end{figure*}

\end{document}